\documentclass{article}
\usepackage{times}
\usepackage{mathtools}
\usepackage{subcaption}
\usepackage{graphicx} %
\usepackage{booktabs} 
\usepackage{amssymb}
\usepackage{times}
\usepackage{epsfig}
\usepackage{graphicx}
\usepackage{color}
\usepackage{url}

\usepackage{natbib}
\usepackage{bbm}
\usepackage{multicol}

\providecommand{\dif}{\mathop{}\!\mathrm d}
\providecommand{\Ex}{\mathbb E}
\providecommand{\hide}[1]{}

\usepackage{amsmath,amsfonts}
\usepackage{amsopn,amssymb,amsthm}

\theoremstyle{plain}
\newtheorem{thm}{Theorem}[section]
\newtheorem{lem}[thm]{Lemma}

\theoremstyle{definition}

\newcommand{\vct}[1]{\boldsymbol{#1}} %
\newcommand{\mat}[1]{\boldsymbol{#1}} %
\newcommand{\field}[1]{\mathbb{#1}}
\newcommand{\R}{\field{R}} %
\newcommand{\T}{^{\textrm T}} %

\newcommand{\ProbOpr}[1]{\mathbb{#1}}
\newcommand{\expect}[2]{%
\ifthenelse{\equal{#2}{}}{\ProbOpr{E}_{#1}}
{\ifthenelse{\equal{#1}{}}{\ProbOpr{E}\left[#2\right]}{\ProbOpr{E}_{#1}\left[#2\right]}}} %
\newcommand{\var}[2]{%
\ifthenelse{\equal{#2}{}}{\ProbOpr{VAR}_{#1}}
{\ifthenelse{\equal{#1}{}}{\ProbOpr{VAR}\left[#2\right]}{\ProbOpr{VAR}_{#1}\left[#2\right]}}} %
\DeclareMathOperator{\argmax}{arg\,max}

\newcommand{\vphi}{\vct{\phi}}
\newcommand{\vpp}{\vct{\phi}}
\newcommand{\vx}{{\vct{x}}}
\newcommand{\vnu}{{\vct{\nu}}}
\newcommand{\vy}{\vct{y}}

\newcommand{\vu}{\vct{u}}

\newcommand{\va}{\vct{a}}

\newcommand{\vr}{\vct{r}}

\newcommand{\vf}{\vct{f}}

\newcommand{\vk}{\vct{k}}

\newcommand{\kk}{\kappa}

\newcommand{\vv}{\vct{v}}

\newcommand{\ww}{\omega}

\newcommand{\mI}{\mat{I}}
\newcommand{\mK}{\mat{K}}

\newcommand{\mSigma}{\mat{\Sigma}}

\newcommand{\rt}{\tilde{r}} %
\newcommand{\bt}{\nu}

\newcommand{\sj}[1]{\textcolor{blue}{SJ: #1}}

\usepackage[accepted]{icml2017}

\usepackage{thm-restate}
\icmltitlerunning{Max-value Entropy Search for Efficient Bayesian Optimization}

\allowdisplaybreaks

\begin{document}
\twocolumn[
\icmltitle{Max-value Entropy Search for Efficient Bayesian Optimization}
\icmlsetsymbol{equal}{*}

\begin{icmlauthorlist}
\icmlauthor{Zi Wang}{mit}
\icmlauthor{Stefanie Jegelka}{mit}
\end{icmlauthorlist}

\icmlaffiliation{mit}{Computer Science and Artificial Intelligence Laboratory, Massachusetts Institute of Technology, Massachusetts, USA}

\icmlcorrespondingauthor{Zi Wang}{ziw@csail.mit.edu}
\icmlcorrespondingauthor{Stefanie Jegelka}{stefje@csail.mit.edu}
\icmlkeywords{Bayesian optimization, Gaussian process, Entropy search}

\vskip 0.3in
]

\printAffiliationsAndNotice{}  %

\begin{abstract} 
Entropy Search (ES) and Predictive Entropy Search (PES) are popular and empirically successful Bayesian Optimization techniques. Both rely on a compelling information-theoretic motivation, and maximize the information gained about the $\arg\max$ of the unknown function; yet, both are plagued by the expensive computation  for estimating entropies. We propose a new criterion, Max-value Entropy Search (MES), that instead uses the information about the maximum function value. We show relations of MES to other Bayesian optimization methods, and establish a regret bound. We observe that MES maintains or improves the good empirical performance of ES/PES, while tremendously lightening the computational burden. In particular, MES is much more robust to the number of samples used for computing the entropy, and hence more efficient for higher dimensional problems. %

\end{abstract}

\section{Introduction}

Bayesian optimization (BO) has become a popular and effective way for black-box optimization of nonconvex, expensive functions in robotics, machine learning, computer vision, and many other areas of science and engineering \citep{brochu2009,calandra2014experimental,krause2011contextual,lizotte2007,snoek2012practical,thornton13,wang17icra}. In BO, a prior is posed on the (unknown) objective function, and the uncertainty given by the associated posterior is the basis for an acquisition function that guides the selection of the next point to query the function. The selection of queries and hence the acquisition function is critical for the success of the method.

Different BO techniques differ in this acquisition function. Among the most popular ones range the Gaussian process upper confidence bound (GP-UCB)~\cite{auer2002b,srinivas2009gaussian}, probability of improvement (PI)~\cite{kushner1964}, and expected improvement (EI)~\cite{mockus1974}. Particularly successful recent additions are entropy search (ES)~\cite{hennig2012} and predictive entropy search (PES) \cite{hernandez2014predictive}, which aim to maximize the mutual information between the queried points and the location of the global optimum.

ES and PES are effective in the sense that they are query-efficient and identify a good point within competitively few iterations, but determining the next query point involves very expensive computations. As a result, these methods are most useful if the black-box function requires a lot of effort to evaluate, and are relatively slow otherwise. Moreover, they rely on estimating the entropy of the $\argmax$ of the function. In high dimensions, this estimation demands a large number of samples from the input space, which can quickly become inefficient.

We propose a twist to the viewpoint of ES and PES that retains the information-theoretic motivation and empirically successful query-efficiency of those methods, but at a much reduced computational cost. The key insight is to replace the uncertainty about the $\argmax$ with the uncertainty about the maximum function value. As a result, we refer to our new method as \emph{Max-value Entropy Search (MES)}. As opposed to the $\argmax$, the maximum function value lives in a one-dimensional space, which greatly facilitates the estimation of the mutual information via sampling. We explore two strategies to make the entropy estimation efficient: an approximation by a Gumbel distribution, and a Monte Carlo approach that uses random features. 

Our contributions are as follows: (1) MES, a variant of the entropy search methods, which enjoys efficient computation and simple implementation; (2) an intuitive analysis which establishes the first connection between ES/PES and the previously proposed criteria GP-UCB, PI and EST~\cite{wang2016est}, where the bridge is formed by MES; (3) a regret bound for a variant of MES, which, to our knowledge, is the first regret bound established for any variant of the entropy search methods; (4) an extension of MES to the high dimensional settings via additive Gaussian processes; and (5) %
empirical evaluations which demonstrate that MES identifies good points as quickly or better than ES/PES, but is much more efficient and robust in estimating the mutual information, and therefore much faster than its input-space counterparts.

After acceptance of this work, we learned that \citet{hoffmanoutput} independently arrived at the acquisition function in Eq.~\eqref{eq:mes}. Yet, our approximation (Eq.~\eqref{eq:mesapprox}) is different, and hence the actual acquisition function we evaluate and analyze is different.

\section{Background}
Our goal is to maximize a black-box function $f:\mathfrak X\rightarrow \R$ where $\mathfrak X \subset \R^{d}$ %
and $\mathfrak X$ is compact. At time step $t$, we select point $\vx_t$ and observe a possibly noisy function evaluation $y_t=f(\vx_t)+\epsilon_t$, where $\epsilon_t \sim \mathcal N(0,\sigma^2)$ are i.i.d.\ Gaussian variables. %
 We use Gaussian processes~\cite{rasmussen2006gaussian} to build a probabilistic model of the black-box function to be optimized. For high dimensional cases, we use a variant of the additive Gaussian process~\cite{duvenaud2011additive,kandasamy2015high}. For completeness, we here introduce some basics of GP and add-GP.  %

\subsection{Gaussian Processes}
\label{ssec:gp}
Gaussian processes (GPs) are distributions over functions, and popular priors for Bayesian nonparametric regression.
In a GP, any finite set of function values has a multivariate Gaussian distribution. A Gaussian process $GP(\mu,k)$ is fully specified by a mean function $\mu(\vx)$ and covariance (kernel) function $k(\vx,\vx')$. 
\hide{
Two frequently used examples of covariance kernel functions are the squared exponential and Mat\'ern kernels.
Let $r=(\vx-\vx')^\top(\vx-\vx')$. Then the squared exponential kernel is 
$$k(\vx,\vx') = \sigma_f^2 \mathrm e^{-\frac{1}{2\ell^2} r},$$ 
with parameters $(\sigma_f,\ell)$. The Mat\'ern kernel is given by 
$$k(\vx,\vx') = \sigma_m^2\frac{2^{1-\xi}}{\Gamma(\xi)} (\frac{\sqrt{2\xi r}}{h})^{\xi}B_\xi(\frac{\sqrt{2\xi r}}{h}),$$ 
where $\Gamma$ is the gamma function, $B_\xi$ is a modified Bessel function, and we have the parameters $\sigma_m$, $h$, and the roughness parameter $\xi$.
}
Let $f$ be a function sampled from $GP(\mu,k)$. %
Given the observations $D_t=\{(\vx_\tau,y_\tau)\}_{\tau=1}^{t}$, we obtain the posterior mean 
$\mu_{t}(\vx) = \vk_t(\vx)\T(\mK_t+\sigma^2\mI)^{-1}\vy_t$
and posterior covariance 
$k_{t}(\vx,\vx') = k(\vx,\vx') - \vk_t(\vx)\T(\mK_t+\sigma^2\mI)^{-1} \vk_t(\vx')$
of the function via the kernel matrix $\mK_t =\left[k(\vx_i,\vx_j)\right]_{\vx_i,\vx_j\in D_t}$ and $\vk_t(\vx) = [k(\vx_i,\vx)]_{\vx_i\in D_t}$~\citep{rasmussen2006gaussian}.
The posterior variance is $\sigma^2_{t}(\vx) = k_t(\vx,\vx)$. %

\subsection{Additive Gaussian Processes}
\label{ssec:addgp}
Additive Gaussian processes (add-GP) were proposed in~\cite{duvenaud2011additive}, and analyzed in the BO setting in~\cite{kandasamy2015high}.  Following the latter, %
we assume that the function $f$ is a sum of independent functions sampled from Gaussian processes that are active on disjoint sets $A_m$ of input dimensions. Precisely,
$f(x) = \sum_{m=1}^M f^{(m)}(x^{A_m})$,
with $A_{i} \cap A_{j} = \emptyset$ for all $i\neq j$, $|\cup_{i=1}^M A_i|=d$, and $f^{(m)}\sim GP(\mu^{(m)}, k^{(m)})$, for all $m \leq M$ ($M \leq d < \infty$). %
As a result of this decomposition, the function $f$ is distributed according to $GP(\sum_{m=1}^M\mu^{(m)}, \sum_{m=1}^M k^{(m)})$. Given a set of noisy observations $D_t=\{(\vx_\tau,y_\tau)\}_{\tau=1}^{t}$ where $y_\tau \sim \mathcal N(f(x_\tau), \sigma^2)$, the posterior mean and covariance of the function component $f^{(m)}$ can be inferred as 
$\mu_{t}^{(m)}(\vx)= \vk^{(m)}_t(\vx)\T(\mK_t+\sigma^2\mI)^{-1}\vy_t$ and 
$k_{t}^{(m)}(\vx,\vx') = k^{(m)}(\vx,\vx') - \vk^{(m)}_t(\vx)\T(\mK_t+\sigma^2\mI)^{-1} \vk^{(m)}_t(\vx')$, where $\vk^{(m)}_t(\vx) = [k^{(m)}(\vx_i,\vx)]_{\vx_i\in D_t}$ and  $\mK_t =\left[ \sum_{m=1}^M k^{(m)}(\vx_i,\vx_j)\right]_{\vx_i,\vx_j\in D_t}$. For simplicity, we use the shorthand $k^{(m)}(\vx,\vx') = k^{(m)}(\vx^{A_m}, \vx'^{A_m})$.

\subsection{Evaluation Criteria}
\label{ssec:eval}
We use two types of evaluation criteria for BO, \emph{simple regret} and \emph{inference regret}. In each iteration, we choose to evaluate one input $\vx_t$ to ``learn'' where the $\argmax$ of the function is. %
The simple regret $r_T = \max_{\vx\in\mathfrak X} f(\vx) - \max_{t\in[1,T]} f(\vx_t)$ measures the value of the best queried point so far. %
After all queries, we may infer an $\argmax$ of the function, which is usually chosen as $\tilde \vx_T = \argmax_{\vx\in\mathfrak X} \mu_{T}(\vx)$~\cite{hennig2012,hernandez2014predictive}. We denote the inference regret as $R_T = \max_{\vx\in\mathfrak X} f(\vx) - f(\tilde x_T)$ which characterizes how satisfying our inference of the $\argmax$ is. %

\section{Max-value Entropy Search}
Entropy search methods use an information-theoretic perspective to select where to evaluate. They
 find a query point that maximizes the information about the location $\vx_* = \argmax_{\vx \in \mathfrak X} f(x)$ whose value $y_* = f(\vx_*)$ achieves the global maximum of the function $f$. Using the negative differential entropy of $p(\vx_* | D_t)$ to characterize the uncertainty about $\vx_*$, ES and PES use the acquisition functions
 \begin{align}
&\alpha_t(x) = 
I(\{\vx,y\}; \vx_*\mid D_t) \\
&\;\;\;\;= H\left(p(\vx_*\mid D_t)\right)- \Ex\left[H(p(\vx_*\mid D_t\cup \{\vx, y\})) \right] \label{eq:es}\\
&\;\;\;\; = H(p(y\mid D_t,\vx))- \Ex\left[H(p(y\mid D_t, \vx, \vx_*)) \right] \label{eq:pes}.
\end{align}
ES uses formulation~\eqref{eq:es}, in which the expectation is over $p(y| D_t, \vx)$, while PES uses the equivalent, symmetric formulation~\eqref{eq:pes}, where the expectation is over $p(\vx_* | D_t)$. Unfortunately, both $p(\vx_* | D_t)$ and its entropy is analytically intractable and have to be approximated via expensive computations.
Moreover, the optimum may not be unique, adding further complexity to this distribution.

We follow the same information-theoretic idea but propose a much cheaper and more robust objective to compute. Instead of measuring the information about the argmax $\vx_*$, we use the information about the \emph{maximum value} $y_* = f(\vx_*)$. 
Our acquisition function is the gain in mutual information between the maximum $y_*$ and the next point we query, which can be approximated analytically by evaluating the entropy of the predictive distribution:
\begin{align}
&\alpha_t(x) = 
I(\{\vx,y\}; y_*\mid D_t) \\
&= H(p(y\mid D_t, \vx)) - \Ex[H(p(y\mid D_t, \vx, y_*)) ] \label{eq:mes} \\
& \approx \frac{1}{K}\sum_{y_*\in Y_*} \left[\frac{\gamma_{y_*}( \vx)\psi(\gamma_{y_*}( \vx))}{2\Psi(\gamma_{y_*}(\vx))} - \log(\Psi(\gamma_{y_*}( \vx))) \right] \label{eq:mesapprox}
\end{align}
where $\psi$ is the probability density function and $\Psi$ the cumulative density function of a normal distribution, and $\gamma_{y_*}(\vx) = \frac{y_* - \mu_t(\vx)}{\sigma_t(\vx)}$. %
The expectation in Eq.~\eqref{eq:mes} is over $p(y_*|D_n)$, which is approximated using Monte Carlo estimation by sampling a set of $K$ function maxima. %
 Notice that the probability in the first term $p(y | D_t, \vx)$ is a Gaussian distribution with mean $\mu_t(\vx)$ and variance $k_t(\vx,\vx)$. The probability in the second term $p(y | D_n, \vx, y_*)$ is a truncated Gaussian distribution: given $y_*$, the distribution of $y$ needs to satisfy $y < y_*$.  Importantly, while ES and PES rely on the expensive, $d$-dimensional distribution $p(\vx_* | D_t)$, here, we use the one-dimensional $p(y_*|D_n)$, which is computationally much easier.

It may not be immediately intuitive that the \emph{value} should bear sufficient information for a good search strategy. Yet, the empirical results in Section~\ref{sec:exp} will demonstrate that this strategy is typically at least as good as ES/PES. From a formal perspective,  \citet{wang2016est} showed how an estimate of the maximum value implies a good search strategy (EST). Indeed, Lemma~\ref{lem:equivalence} will make the relation between EST and a simpler, degenerate version of MES explicit.

Hence, it remains to determine how to sample $y_*$. We propose two strategies: (1) sampling from an approximation via a Gumbel distribution; and (2) sampling functions from the posterior Gaussian distribution and maximizing the functions to obtain samples of $y_*$. We present the MES algorithm in Alg.~\ref{alg:mes}.%

\begin{algorithm} %
  \caption{Max-value Entropy Search (MES)%
  }\label{alg:mes}
  \begin{algorithmic}[1]
    \FUNCTION{MES\,($f, D_0$)}
      \FOR{$t = 1,\cdots, T $}
      \STATE $\alpha_{t-1}(\cdot)\gets$\textsc{Approx-MI\,}($D_{t-1}$)%
      \STATE $\vx_t\gets \argmax_{\vx\in\mathfrak X}{\alpha_{t-1}(\vx)}$ 
      \STATE $y_t\gets f(\vx_t) + \epsilon_t, \epsilon_t\sim\mathcal N(0,\sigma^2)$
      \STATE $\mathfrak D_t \gets D_{t-1}\cup \{\vx_t,y_t\}$
      \ENDFOR
      \ENDFUNCTION
      \item[]
    \FUNCTION{Approx-MI\,($D_t$)}%
    \IF{Sample with Gumbel}
    \STATE approximate $\Pr[\hat{y}_* < y]$ with $\mathcal G(a,b)$
    \STATE sample a $K$-length vector $\vr\sim \text{Unif}([0,1])$
    \STATE $\vy_* \gets a-b\log(-\log \vr)$
    \ELSE
    \FOR{$i=1,\cdots,K$}
    \STATE sample $\tilde f \sim GP(\mu_t,k_t \mid D_t)$
    \STATE $y_{*(i)} \gets \max_{\vx\in\mathfrak X}{\tilde f(\vx)}$
    \ENDFOR 
    \STATE $\vy_* \gets [y_{*(i)}]_{i=1}^K$
    \ENDIF
    \STATE \textbf{return} $\alpha_t(\cdot)$ in Eq.~\eqref{eq:mesapprox}
    \ENDFUNCTION
  \end{algorithmic}
\end{algorithm}

\subsection{Gumbel Sampling} %
\label{spec:gumbel}
The marginal distribution of $f(x)$ for any $x$ is a one-dimensional Gaussian, and hence the distribution of $y^*$ may be viewed as the maximum of an infinite collection of dependent Gaussian random variables.
Since this distribution is difficult to compute, we make two simplifications. First, we replace the continuous set $\mathfrak X$ by a discrete (finite), dense subset $\hat{\mathfrak X}$ of representative points. If we select $\hat{\mathfrak X}$ to be an $\epsilon$-cover of $\mathfrak X$ and the function $f$ is Lipschitz continuous with constant $L$, then we obtain a valid upper bound on $f(\mathfrak X)$ by adding $\epsilon L$ to any upper bound on $f(\hat{\mathfrak X})$.

Second, we use a ``mean field'' approximation and treat the function values at the points in $\hat{\mathfrak X}$ as independent. This approximation tends to over-estimate the maximum; this follows from Slepian's lemma if $k(x,x') \geq 0$. Such upper bounds still lead to optimization strategies with vanishing regret, whereas lower bounds may not \citep{wang2016est}.

We sample from the approximation $\hat{p}(y^*|D_n)$ via its cumulative distribution function (CDF) $\widehat{\Pr}[y_* < z] = \prod_{\vx\in\hat{\mathfrak X}} \Psi(\gamma_z(\vx))$. That means we sample $r$ uniformly from $[0,1]$ and find $z$ such that $\Pr[y_* < z] = r$. A binary search for $z$ to accuracy $\delta$ requires $O(\log\tfrac1\delta)$ queries to the CDF, and each query takes $O(|\hat{\mathfrak X}|)\approx O(n^d)$ time, so we obtain an overall time of $O(M|\hat{\mathfrak X}| \log \frac{1}{\delta})$ for drawing $M$ samples.

To sample more efficiently, we propose a $O(M+|\hat{\mathfrak X}|\log \frac{1}{\delta})$-time strategy, by approximating the CDF by a Gumbel distribution: $\widehat{\Pr}[y_* < z]\approx \mathcal G(a,b) = e^{-e^{-\frac{z-a}{b}}}$. This choice is motivated by the Fisher-Tippett-Gnedenko theorem~\cite{fisher1930genetical}, which states that the maximum of a set of i.i.d. Gaussian variables is asymptotically described by a Gumbel distribution (see the appendix for further details). This does not in general extend to non-i.i.d. Gaussian variables, but we nevertheless observe that in practice, this approach yields a good and fast approximation.

We sample from the Gumbel distribution via the Gumbel quantile function: we sample $r$ uniformly from $[0,1]$, and let the sample be $y = \mathcal G^{-1}(a,b) = a-b\log(-\log r)$. We set the appropriate Gumbel distribution parameters $a$ and $b$ by percentile matching and solve the two-variable linear equations $a-b\log(-\log r_1) = y_1$ and $a-b\log(-\log r_2) = y_2$, where $\Pr[{y}_* < y_1]=r_1$ and $\Pr[{y}_* < y_2]=r_2$. In practice, we use $r_1=0.25$ and $r_2=0.75$ so that the scale of the approximated Gumbel distribution
is proportional to the interquartile range of the CDF $\hat{\Pr}[y_* < z]$. %

\subsection{Sampling $y_*$ via Posterior Functions}
For an alternative sampling strategy we follow~\cite{hernandez2014predictive}: we draw functions from the posterior GP and then maximize each of the sampled functions.
Given the observations $D_t = \{(\vx_\tau, y_\tau)_{\tau=1}^t\}$, we can approximate the posterior Gaussian process using a 1-hidden-layer neural network $\tilde f(\vx) = \va_t\T \vphi(\vx)$ where $\vphi(x)\in \R^D$ is a vector of feature functions~\cite{neal96,rahimi2007random}
and the Gaussian weight $\va_t\in\R^D$ is distributed according to a multivariate Gaussian $\mathcal N(\vnu_t, \mSigma_t)$.

\textbf{Computing $\vphi(x)$.} By Bochner's theorem~\cite{rudin2011fourier}, the Fourier transform $\hat k$ of a continuous and translation-invariant kernel $k$ is guaranteed to be a probability distribution. Hence we can write the kernel of the GP to be $k(\vx,\vx') = \Ex_{\ww\sim\hat k(\ww)}[e^{i\ww\T (\vx-\vx')}] =\Ex_{c\sim U[0,2\pi]} \Ex_{\hat k} [2 \cos(\ww\T \vx + c)\cos(\ww\T \vx' + c)]$ and approximate the expectation by $k(\vx,\vx')\approx \vpp\T(\vx)\vpp(\vx')$ where $\phi_i(\vx) = \sqrt{\frac2D} \cos(\ww_i\T \vx + c_i)$, $\ww_i\sim \hat \kk(\ww)$, and $c_i \sim U[0,2\pi]$ for $i = 1,\dots, D$.

\textbf{Computing $\vnu_t, \mSigma_t$.} By writing the GP as a random linear combination of feature functions $\va^T_t\vphi(\vx)$, we are defining the mean and covariance of the GP to be $\mu_t(\vx) = \vnu\T \vphi(\vx)$ and $k(\vx,\vx') = \vphi(\vx)\T\mSigma_t\vphi(\vx')$. Let $Z = [z_1, \cdots, z_t] \in \R^{D\times t}$, where $z_\tau \coloneqq \vpp (\vx_\tau) \in \R^D$. 
The GP posterior mean and covariance in Section~\ref{ssec:gp} become
$\mu_t(\vx) = z\T Z (Z\T Z + \sigma^2 \mI)^{-1} \vy_t$ and 
$k_t (\vx, \vx')= z\T z' - z\T Z (Z\T Z + \sigma^2 \mI)^{-1} Z\T z'$. 
Because $Z(Z\T Z + \sigma^2 \mI)^{-1} = (ZZ\T +\sigma^2 \mI)^{-1}Z$, we can simplify the above equations and obtain $\vnu_t = \sigma^{-2}\mSigma_t Z_t\vy_t $ and $\mSigma_t = (Z Z\T \sigma^{-2} +\mI)^{-1}$.

To sample a function from this random 1-hidden-layer neural network, we sample $\tilde \va$ from $\mathcal N(\vnu_t, \mSigma_t)$ and construct the sampled function $\tilde f = \tilde \va\T \vphi(\vx)$. Then we optimize $\tilde f$ with respect to its input to get a sample of the maximum of the function $\max_{\vx\in\mathfrak X} \tilde f(\vx)$.

\hide{
We first review Mercer's theorem~\cite{mercer1909functions}, which builds the foundation of kernel methods.
\begin{thm}[Mercer~\cite{mercer1909functions}]
Any kernel $\kk:\mathfrak X \times \mathfrak X \rightarrow \R$ satisfying $\int \kk(x,x')f(x)f(x')\dif x \dif x' \geq 0$ for all $L_2(\mathfrak X)$ measurable functions $f$ can be expanded into 
\begin{align}
\kk(x,x') = \sum_j \lambda^{(j)} \phi_j(x)\phi_j(x'), 
\end{align}
 where $\lambda_j$ and $\phi_j$ are orthonormal on $L_2(\mathfrak X)$.

\end{thm}
Without loss of generality, we assume $\sum_j \lambda^{(j)} = 1, p(\lambda^{(j)}) = \lambda^{(j)}$. We can get a low variance approximation of $\kk(x,x')$ by sampling $\lambda\sim p(\lambda)$ and the corresponding random feature $\phi_\lambda$.
\begin{align}
\kk(x,x') = \Ex_{\lambda}[\phi_\lambda(x)\phi_{\lambda}(x')] \approx \frac{1}{D} \sum_{i=1}^D \phi_{\lambda_i}(x)\phi_{\lambda_i}(x') = \frac{1}{D}\vpp(x)\T\vpp(x'), 
\end{align}

where $\vpp(x)\in \R^D$ is the random feature vector for $x$. Let $z \coloneqq \frac{1}{\sqrt{D}} \vpp (x), z_\tau \coloneqq \frac{1}{\sqrt{D}} \vpp (x_\tau) \in \R^D$ and $Z = [z_1, \cdots, z_t] \in \R^{D\times t}$.

The GP posterior mean and covariance become
\begin{align}
\mu_t(x) &= z\T Z (Z\T Z + \sigma^2 \mI)^{-1} \vy_t, \label{eq:gpprimal1}\\
\kk_t (x)&= z\T z - z\T Z (Z\T Z + \sigma^2 \mI)^{-1} Z\T z. \label{eq:gpprimal2}
\end{align}

Because $Z(Z\T Z + \sigma^2 \mI)^{-1} = (ZZ\T +\sigma^2 \mI)^{-1}Z$, Eq.~\eqref{eq:gpprimal1} and Eq.~\eqref{eq:gpprimal2} can be simplified to

\begin{align}
\Sigma_t & = (Z Z\T \sigma^{-2} +\mI)^{-1}\\
\mu_t(x) &= \sigma^{-2}z\T \Sigma_t Z \vy_t  \\
\kk_t(x) &= z\T \Sigma_t z %
\end{align}
}

\subsection{Relation to Other BO Methods}
As a side effect, our new acquisition function draws connections between ES/PES and other popular BO methods. The connection between MES and ES/PES follows from the information-theoretic viewpoint; the following lemma makes the connections to other methods explicit.
\begin{lem}\label{lem:equivalence}
  The following methods are equivalent:
  \vspace{-5pt}
\begin{enumerate}\setlength{\itemsep}{-1pt}
\item MES, where we only use a single sample $y_*$ for $\alpha_t(x)$;
\item EST with $m = y_*$;
\item GP-UCB with $\beta^{\frac12} = \min_{\vx\in\mathfrak X} \frac{y_*-\mu_{t} (\vx)}{\sigma_{t} (\vx)}$;
\item PI with $\theta=y_*$.
\end{enumerate}
\end{lem}
This equivalence no longer holds if we use $M>1$ samples of $y_*$ in MES.
\begin{proof}
  The equivalence among 2,3,4 is stated in Lemma 2.1 in~\cite{wang2016est}. What remains to be shown is the equivalence between 1 and 2. When using a single $y_*$ in MES,
  the next point to evaluate is chosen by maximizing $\alpha_t(\vx) =\gamma_{y_*}( \vx)\frac{\psi(\gamma_{y_*}( \vx))}{2\Psi(\gamma_{y_*}(\vx))} - \log(\Psi(\gamma_{y_*}( \vx)))$ and $\gamma_{y_*} = \frac{y_* - \mu_t(\vx)}{\sigma_t(\vx)}$. For EST with $m=y_*$, the next point to evaluate is chosen by minimizing $\gamma_{y_*}( \vx)$. Let us define a function $g(u) = u\frac{\psi(u)}{2\Psi(u)} - \log(\Psi(u))$. Clearly, $\alpha_t(\vx) = g(\gamma_{y_*}(\vx))$. Because $g(u)$ is a monotonically decreasing function, maximizing $g(\gamma_{y_*}(\vx))$ is equivalent to minimizing $\gamma_{y_*}(\vx)$. Hence 1 and 2 are equivalent.
\end{proof}

\subsection{Regret Bound}
The connection with EST directly leads to a bound on the simple regret of MES, when using only one sample of $y_*$. We prove Theorem~\ref{thm:regret} in the appendix.
\begin{restatable}[Simple Regret Bound]{thm}{regret}
\label{thm:regret}
Let $F$ be the cumulative probability distribution for the maximum of any function $f$ sampled from $GP(\mu, k)$ over the compact search space $\mathfrak X \subset R^{d}$, where $k(\vx,\vx')\leq 1,\forall \vx,\vx'\in\mathfrak X$.
Let $f_*=\max_{\vx\in \mathfrak X} f(\vx)$ 
 and $w = F(f_*) \in (0,1)$, and assume the observation noise is iid $\mathcal N(0,\sigma)$.
If in each iteration $t$, the query point is chosen as $\vx_t = \argmax_{\vx\in\mathfrak X} \gamma_{y^t_*}( \vx)\frac{\psi(\gamma_{y^t_*}( \vx))}{2\Psi(\gamma_{y^t_*}(\vx))} - \log(\Psi(\gamma_{y^t_*}( \vx)))$, where $\gamma_{y^t_*}(\vx) = \frac{y^t_* - \mu_t(\vx)}{\sigma_t(\vx)}$ and $y^t_*$ is drawn from $F$, then with probability at least $1-\delta$, in $T' = \sum_{i=1}^T\log_{w} \frac{\delta}{2\pi_i}$ number of iterations, the simple regret satisfies
\begin{align}
\label{eq:regret}
r_{T'} \leq  \sqrt{\frac{C \rho_T}{T}} (\bt_{t^*} + \zeta_T)
\end{align}
where $C = 2/\log (1+\sigma^{-2})$ and $\zeta_T=(2\log(\frac{\pi_T}{\delta}))^{\frac12}$; $\pi$ satisfies $\sum_{i=1}^T \pi_i^{-1} \leq 1$ and $\pi_t > 0$, and $t^*=\argmax_t \bt_t$ with $\bt_t \triangleq \min_{\vx\in\mathfrak X, y^t_* > f_*} \gamma_{y^t_*}(\vx)$, and $\rho_T$ is the maximum information gain of at most $T$ selected points.
\end{restatable}
\hide{
\sj{I do not understand the following discussion. The reader does not know the bounds for EST or UCB unless you state them, and not sure anyone would know the effect of a large or small $y_*$ without reading your other paper. Maybe cut this, or make it more high-level, or discuss what we would expect for using more than one $y_*$}
At first sight, it might seem like MES with a point estimate does not have a converging rate as good as $EST$ or $GP-UCB$. However, notice that $\min_{\vx\in\mathfrak X}\gamma_{y_1}(\vx) < \min{\vx\in\mathfrak X} \gamma_{y_2}(\vx)$ if $y_1 < y_2$, which decides the rate of convergence in Eq.~\ref{eq:regret}. So if we use $y_*$ that is too large, the regret bound could be worse. If we use $y_*$ that is smaller than $f_*$, however, its value won't count towards the learning regret in our proof, so it is also bad for the regret upper bound. With no principled way of setting $y_*$ since $f_*$ is unknown. Our regret bound in Theorem~\ref{thm:regret} is a randomized trade-off between sampling large and small $y_*$.
}

\subsection{Model Adaptation}
\label{sec:model}

In practice we do not know the hyper-parameters of the GP, so we must adapt our GP model as we observe more data. 
A standard way to learn the GP hyper-parameters is to optimize the marginal data likelihood with respect to the hyper-parameters. As a full Bayesian treatment, we can also draw samples of the hyper-parameters using slice sampling~\cite{vanhatalo2013gpstuff}, and then marginalize out the hyper-parameters in our acquisition function in Eq.~\eqref{eq:mesapprox}. Namely, if we use %
$E$ to denote the set of sampled settings for the GP hyper-parameters, our acquisition function becomes
\begin{align*}
\alpha_t(x) = \sum_{\eta\in E}\sum_{y_*\in Y_*}\left[ \frac{\gamma^{\eta}_{y_*}( \vx)\psi(\gamma^{\eta}_{y_*}( \vx))}{2\Psi(\gamma^{\eta}_{y_*}(\vx))} - \log(\Psi(\gamma^{\eta}_{y_*}( \vx)))\right],
\end{align*} %
where $\gamma^{\eta}_{y_*}(\vx) = \frac{y_* - \mu^{\eta}_t(\vx)}{\sigma^{\eta}_t(\vx)}$ and the posterior inference on the mean function $\mu_t^{\eta}$ and $\sigma_t^{\eta}$  depends on the GP hyper-parameter setting $\eta$. Similar approaches have been used in~\cite{hernandez2014predictive, snoek2012practical}.

\hide{
For learning additive structure, we follow~\cite{newicml} and construct a Bayesian model for the additive structure where each input dimension is assigned to a unique group. The generative processes of the group assignment is as follows, 1) draw group mixing proportion $\theta\sim \textsc{Dir}(\alpha)$; 2) for each dimension $j$, choose a group assignment $w_j \sim \textsc{Multi}(\theta)$. Our function is $f(x) =\sum_{m=1}^M f^{(i)}(x^{A_m})$, where $A_{m} = \{{m: w_j=m)}\}$ is the set of active dimensions for function $f^{(i)}$. We integrate out $\theta$ and use Gibbs sampling to learn the assignments $w_j$. For each input dimension, we sample the group assignment from the conditional distribution 
$p(w_j = m\mid w_{\neg j}, D_t;\alpha) \propto e^{\phi_m}$
where $\phi_m = -\frac12 \vy\T (K_m+\sigma^2 \mI)^{-1}\vy -\frac12\log |K_m+\sigma^2 \mI| + 
\log (|A_m| + \alpha_m),$ and $K_m$ is the gram matrix associated with $D_t$ by setting $w_j = m$. We iterate between learning hyper-parameters and additive model to obtain a good setting of the additive structure and hyper-parameters.
}

\section{High Dimensional MES with Add-GP}

The high-dimensional input setting has been a challenge for many BO methods. 
We extend MES to this setting via additive Gaussian processes (Add-GP). In the past, Add-GP has been used and analyzed for GP-UCB~\cite{kandasamy2015high}, which assumed the high dimensional black-box function is a summation of several disjoint lower dimensional functions. Utilizing this special additive structure, we overcome the statistical problem of having insufficient data to recover a complex function, and the difficulty of optimizing acquisition functions in high dimensions.

Since the function components $f^{(m)}$ are independent, %
we can maximize the mutual information between the input in the active dimensions $A_m$ and maximum of $f^{(m)}$ for each component separately. Hence, we have a separate acquisition function for each component, where $y^{(m)}$ is the evaluation of $f^{(m)}$:
\begin{align}
&\alpha^{(m)}_t(\vx) = 
I(\{\vx^{A_m},y^{(m)}\}; y^{(m)}_*\mid D_t) \\
&= H(p(y^{(m)}\mid D_t, \vx^{A_m})) \nonumber \\
&\;\;\;\;- \Ex[H(p(y^{(m)}\mid D_t, \vx^{A_m}, y^{(m)}_*)) ] \label{eq:addmes} \\
& \approx \sum_{y_*^{(m)}} \gamma^{(m)}_{y_*}( \vx)\frac{\psi(\gamma^{(m)}_{y_*}( \vx))}{2\Psi(\gamma^{(m)}_{y_*}(\vx))}  - \log(\Psi(\gamma^{(m)}_{y_*}( \vx))) \label{eq:addmesapprox}
\end{align}
where $\gamma^{(m)}_{y_*}(\vx) = \frac{y^{(m)}_* - \mu^{(m)}_t(\vx)}{\sigma^{(m)}_t(\vx)}$. Analogously to the non-additive case, we sample $y^{(m)}_*$, separately for each function component. %
We select the final $x_t$ by choosing a sub-vector $x_t^{(m)} \in \arg\max_{\vx^{(m)} \in A_m} \alpha^{(m)}_t(\vx^{(m)})$ and concatenating the components.

\paragraph{Sampling $y^{(m)}_*$ with a Gumbel distribution.}
The Gumbel sampling from Section~\ref{spec:gumbel} directly extends to sampling $y^{(m)}_*$, approximately. We simply need to sample from the component-wise CDF $\widehat{\Pr}[{y}^{(m)}_* < z] = \prod_{\vx\in\hat{\mathfrak X}} \Psi(\gamma^{(m)}_y(\vx)))$, and use the same Gumbel approximation.

\paragraph{Sampling $y^{(m)}_*$ via posterior functions.} The additive structure removes some connections on the input-to-hidden layer of our 1-hidden-layer neural network approximation $\tilde f(\vx)=\va_t\T\vphi(\vx)$.  Namely, for each feature function $\phi$ there exists a unique group $m$ such that $\phi$ is only active on $\vx^{A_m}$, %
 and $\phi(\vx) =  \sqrt{\frac2D} \cos(\ww\T \vx^{A_m} + c)$ where $\R^{|A_m|}\ni\ww\sim \hat \kk^{(m)}(\ww)$ and $c \sim U[0,2\pi]$. Similar to the non-additive case, we may draw a posterior sample $\va_t\sim\mathcal N(\vnu_t,\mSigma_t)$ where $\vnu_t = \sigma^{-2}\mSigma_t Z_t\vy_t $ and $\mSigma_t = (Z Z\T \sigma^{-2} +\mI)^{-1}$. Let $B_m = \{i: \vphi_i(\vx) \text{ is active on $\vx^{A_m}$}\}$. The posterior sample for the function component $f^{(m)}$ is $\tilde f^{(m)}(\vx)=(\va^{B_m}_t)\T\phi^{B_m}(\vx^{A_m})$. Then we can maximize $\tilde f^{(m)}$ to obtain a sample for $y^{(m)}_*$.

The algorithm for the additive max-value entropy search method (add-MES) is shown in Algorithm~\ref{alg:addgpmes}. The function $\textsc{Approx-MI}$ does the pre-computation for approximating the mutual information in a similar way as in Algorithm~\ref{alg:mes}, except that it only acts on the active dimensions in the $m$-th group. %

\begin{algorithm} %
  \caption{Additive Max-value Entropy Search}\label{alg:addgpmes}
  \begin{algorithmic}[1]
    \FUNCTION{Add-MES\,($f, D_0$)}
      \FOR{$t = 1,\cdots, T $}
      \FOR{$m=1,\cdots,M$}
      \STATE $\alpha^{(m)}_{t-1}(\cdot)\gets$\textsc{Approx-MI\,}($D_{t-1}$)
      \STATE $\vx^{A_m}_t\gets \argmax_{\vx^{A_m}\in\mathfrak X^{A_m}}{\alpha_{t-1}^{(m)}(\vx)}$ 
      \ENDFOR
      \STATE $y_t\gets f(\vx_t) + \epsilon_t, \epsilon_t\sim\mathcal N(0,\sigma^2)$
      \STATE $\mathfrak D_t \gets D_{t-1}\cup \{\vx_t,y_t\}$
      \ENDFOR
      \ENDFUNCTION
  \end{algorithmic}
\end{algorithm}

\hide{
\subsection{Regret analysis}

\begin{thm}
Assume the function $f$ is structured according to Section~\ref{ssec:addgp}. For each function component $f^{(m)}$, let $\hat{\mathfrak X}^{A_m}$ be an $\epsilon$-covering of $\mathfrak X^{A_m}$. Assume there exists $a,b>0$ such that $\Pr[|f^{(m)}(\vx^{A_m}) - f^{(m)}(\vx'^{A_m})|\leq L\epsilon , \forall ||\vx^{A_m}-\vx'^{A_m}||\leq \epsilon]\geq 1-ae^{-L^2/b^2}.$
Let $\sigma^2$ be the variance of the Gaussian noise in the observation,
 $C = 2/\log (1+\sigma^{-2})$,  $\gamma_T$ the maximum information gain of the selected points, and $t^*=\argmax_t \bt_t$ where $\bt_t \triangleq \min_{\vx\in\mathfrak X} \frac{y^{(m)}_*+\epsilon L-\mu^{(m)}_{t-1}(\vx)}{\sigma^{(m)}_{t-1}(\vx)}$, and $y^{(m)}$ is sampled from $\Pr[y^{(m)}_* < y] = \prod_{\vx\in\hat{\mathfrak X}^{A_m}} \Psi(\gamma^{(m)}_y(\vx)))$. 
With probability at least $1-\delta$, the learning regret up to time step $T$ is bounded as
\begin{equation*}
   r_t\leq \sqrt{\frac{CM \gamma_T}{T}} (\bt_{t^*} + \zeta_T), 
\end{equation*}
where $\zeta_T=(2\log(\frac{\pi_t}{\delta}))^{\frac12}$, $\sum_{t=1}^T \pi_t^{-1} \leq 1$, $\pi_t > 0$, and $L=b\sqrt{\log\frac{2a}{\delta}}$.
\end{thm}
}

\section{Experiments}
\label{sec:exp}
In this section, we probe the empirical performance of MES and add-MES on a variety of tasks. Here, MES-G denotes MES with $y_*$ sampled from the approximate Gumbel distribution, and MES-R denotes MES with $y_*$ computed by maximizing a sampled function represented by random features. Following~\cite{hennig2012,hernandez2014predictive}, we adopt the zero mean function and non-isotropic squared exponential kernel as the prior for the GP. We compare to methods from the entropy search family, i.e., ES and PES, and to other popular Bayesian optimization methods including GP-UCB (denoted by UCB), PI, EI and EST. The parameter for GP-UCB was set according to Theorem 2 in~\cite{srinivas2009gaussian};
the parameter for PI was set to be the observation noise $\sigma$. For the functions with unknown GP hyper-parameters, every 10 iterations, we learn the GP hyper-parameters using the same approach as was used by PES~\cite{hernandez2014predictive}. For the high dimensional tasks, we follow~\cite{kandasamy2015high} and sample the additive structure/GP parameters with the highest data likelihood when they are unknown. We evaluate performance according to the simple regret and inference regret as defined in Section~\ref{ssec:eval}. We used the open source Matlab implementation of PES, ES and EST~\cite{hennig2012,hernandez2014predictive,wang2016est}.
 Our Matlab code and test functions are available at \url{https://github.com/zi-w/Max-value-Entropy-Search/}.
\subsection{Synthetic Functions}
\begin{figure}
\centering
\includegraphics[width=.5\textwidth]{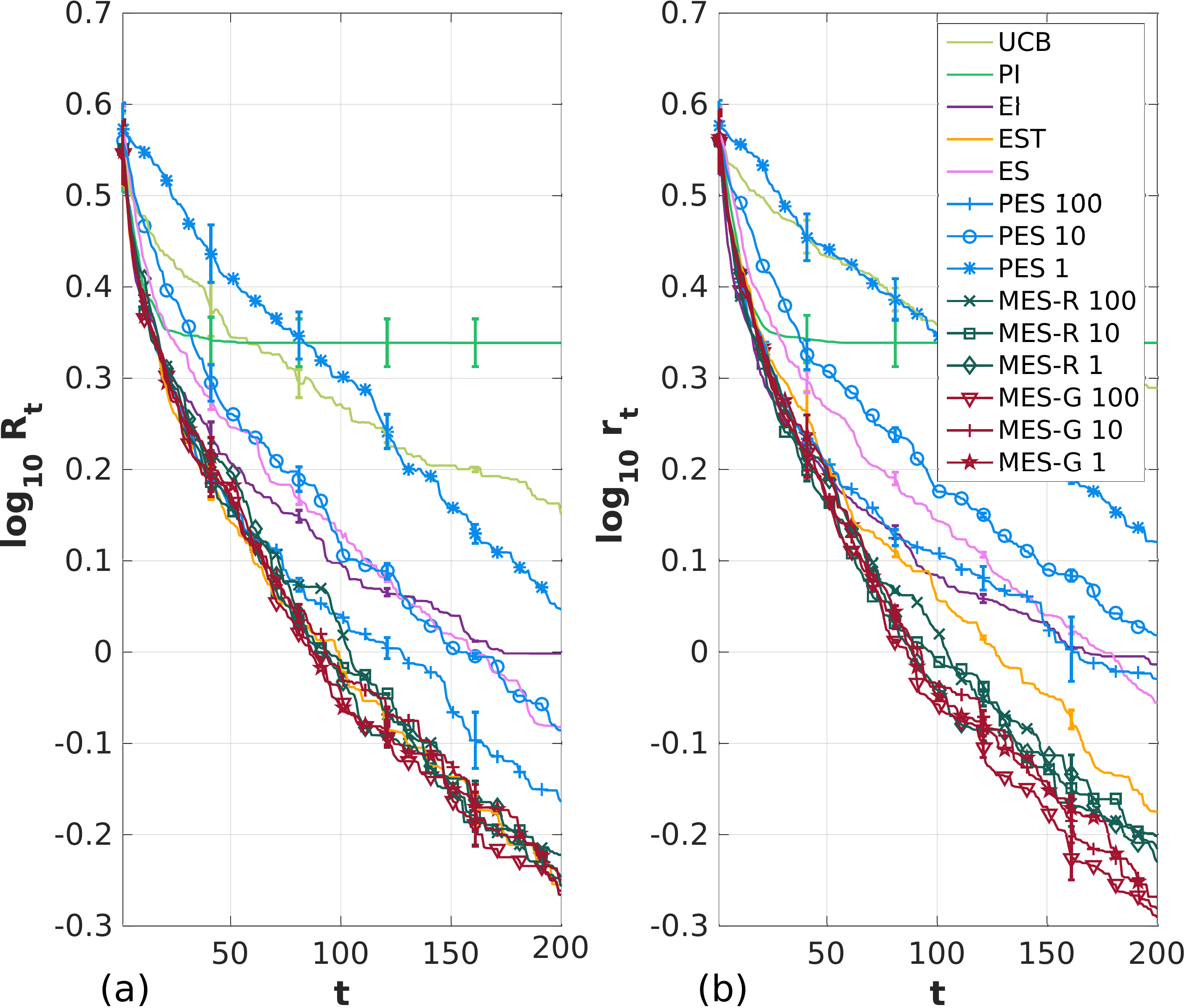}
\caption{(a) Inference regret; (b) simple regret. MES methods are much less sensitive to the number of maxima $y_*$ sampled for the acquisition function (1, 10 or 100) than PES is to the number of argmaxes $x_*$. }
\label{fig:synth1}
\end{figure}

\begin{table}[t]
\caption{The runtime of selecting the next input. PES 100 is significantly slower than other methods. MES-G's runtime is comparable to the fastest method EI while it performs better in terms of simple and inference regrets.}
\label{tb:botiming}
\vskip -0.2in
\begin{center}
\begin{small}
\begin{sc}
\begin{tabular}{lc|lc}%
\hline
\abovespace\belowspace
Method & Time (s) & Method & Time (s) \\
\hline
\abovespace
UCB    & {\color{red}$0.08\pm0.05$}  & PES 1& $0.20\pm 0.06$ \\
PI & {\color{red}$0.10\pm0.02 $}        & MES-R 100& $5.85\pm 0.86 $\\
EI    & {\color{red}$0.07\pm 0.03$}             & MES-R 10 & $0.67\pm 0.11$ \\
EST    & $0.15\pm0.02$    & MES-R 1&  $0.13\pm 0.03$       \\
ES    & $8.07\pm3.02$      &  MES-G 100 & $0.12\pm 0.02$\\
PES 100   & $15.24\pm4.44$    & MES-G 10 & {\color{red}$0.09\pm0.02 $} \\
PES 10    & $1.61\pm 0.50$      & MES-G 1&    {\color{red}$0.09\pm0.03$}     \\
\hline
\end{tabular}
\end{sc}
\end{small}
\end{center}
\vskip -0.2in
\end{table}

We begin with a comparison on synthetic functions sampled from a 3-dimensional GP, to probe our conjecture that MES is much more robust to the number of $y_*$ sampled to estimate the acquisition function than PES is to the number of $x_*$ samples.
For PES, we sample 100 (PES 100), 10 (PES 10) and 1 (PES 1) argmaxes for the acquisition function. Similarly, we sample 100, 10, 1 $y_*$ values for MES-R and MES-G. We average the results on 100 functions sampled from the same Gaussian kernel with scale parameter $5.0$ and bandwidth parameter $0.0625$, and
observation noise $\mathcal{N}(0,0.01^2)$.

Figure~\ref{fig:synth1} shows the simple and inference regrets. 
For both regret measures, PES is very sensitive to the the number of $x_*$ sampled for the acquisition function: 100 samples lead to much better results than 10 or 1. In contrast, both MES-G and MES-R perform competitively even with 1 or 10 samples.
Overall, MES-G is slightly better than MES-R, and both MES methods performed better than other ES methods. 
MES methods performed better than all other methods with respect to simple regret. For inference regret, MES methods performed similarly to EST, and much better than all other methods including PES and ES. %

In Table~\ref{tb:botiming}, we show the runtime of selecting the next input per iteration\footnote{All the timing experiments were run exclusively on an Intel(R) Xeon(R) CPU E5-2680 v4 @ 2.40GHz. The function evaluation time is excluded.} using GP-UCB, PI, EI, EST, ES, PES, MES-R and MES-G on the synthetic data with fixed GP hyper-parameters. For PES and MES-R, every $x_*$ or $y_*$ requires running an optimization sub-procedure, so their running time grows noticeably with the number of samples. MES-G avoids this optimization, and competes with the fastest methods EI and UCB.

In the following experiments, we set the number of $x_*$ sampled for PES to be 200, and the number of $y_*$ sampled for MES-R and MES-G to be 100 unless otherwise mentioned. 

\subsection{Optimization Test Functions}
\hide{
\begin{figure*}
\centering
\includegraphics[width=0.85\textwidth]{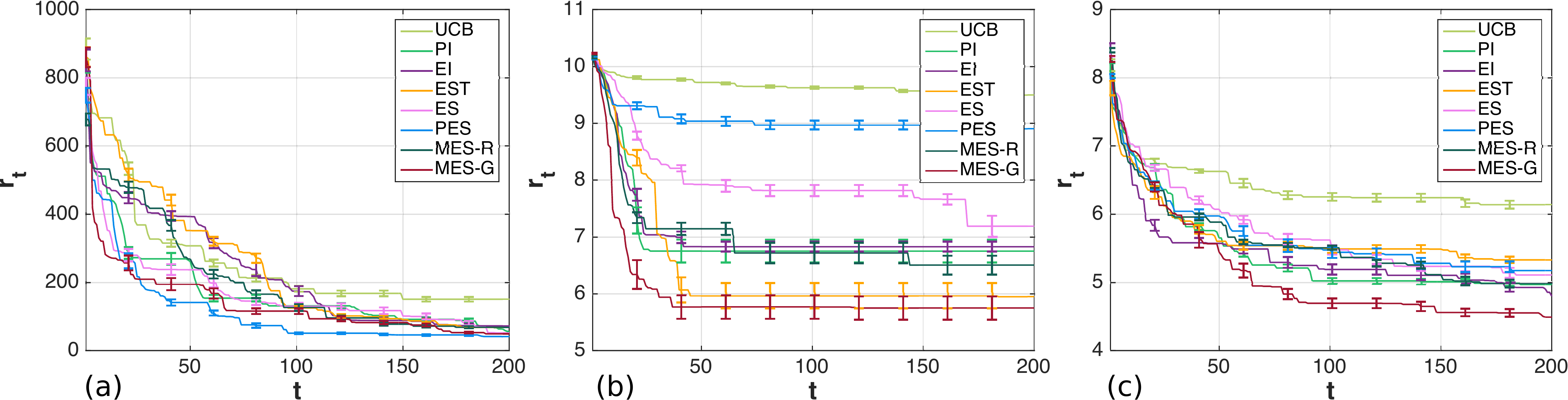}
\caption{(a) 2-d eggholder function; (b) 10-d Shekel function; (c) 10-d Michalewicz function. PES achieves lower regret on the 2-d function while MES-G performed better than other methods on the two 10-d optimization test functions.}
\label{fig:test}
\end{figure*}
}
We test on three challenging optimization test functions: the 2-dimensional eggholder function, the 10-dimensional Shekel function and the 10-dimensional Michalewicz function. All of these functions have many local optima. We randomly sample 1000 points to learn a good GP hyper-parameter setting, and then run the BO methods with the same hyper-parameters. The first observation is the same for all methods.
We repeat the experiments 10 times. The averaged simple regret is shown in the appendix%
, and the inference regret is shown in Table~\ref{tb:test}. On the 2-d eggholder function, PES was able to achieve better function values faster than all other methods, which verified the good performance of PES when sufficiently many $x_*$ are sampled. However, for higher-dimensional test functions, the 10-d Shekel and 10-d Michalewicz function, MES methods performed much better than PES and ES, and MES-G performed better than all other methods. 

\begin{table}[t]
\caption{Inference regret $R_T$ for optimizing the eggholder function, Shekel function, and Michalewicz function.}
\label{tb:test}
\vskip -0.5in
\begin{center}
\begin{small}
\begin{sc}
\begin{tabular}{llll}%
\hline
\abovespace\belowspace
Method & Eggholder & Shekel & Michalewicz \\
\hline
\abovespace
UCB & $141.00 \pm 70.96$ & $9.40 \pm 0.26$ & $6.07 \pm 0.53$   \\
PI & $52.04 \pm 39.03$ & $6.64 \pm 2.00$ & $4.97 \pm 0.39$   \\
EI & $71.18 \pm 59.18$ & $6.63 \pm 0.87$ & $4.80 \pm 0.60$   \\
EST & $55.84 \pm 24.85$ & $5.57 \pm 2.56$ & $5.33 \pm 0.46$   \\
ES & $48.85 \pm 29.11$ & $6.43 \pm 2.73$ & $5.11 \pm 0.73$   \\
PES & {\color{red}$37.94 \pm 26.05$} & $8.73 \pm 0.67$ & $5.17 \pm 0.74$   \\
MES-R & $54.47 \pm 37.71$ & $6.17 \pm 1.80$ & $4.97 \pm 0.59$   \\
\belowspace
MES-G & $46.56 \pm 27.05$ & {\color{red}$5.45 \pm 2.07$} & {\color{red}$4.49 \pm 0.51$}   \\
\hline
\end{tabular}
\end{sc}
\end{small}
\end{center}
\vskip -0.2in
\end{table}

\subsection{Tuning Hyper-parameters for Neural Networks}
\begin{figure}
\centering
\includegraphics[width=0.45\textwidth]{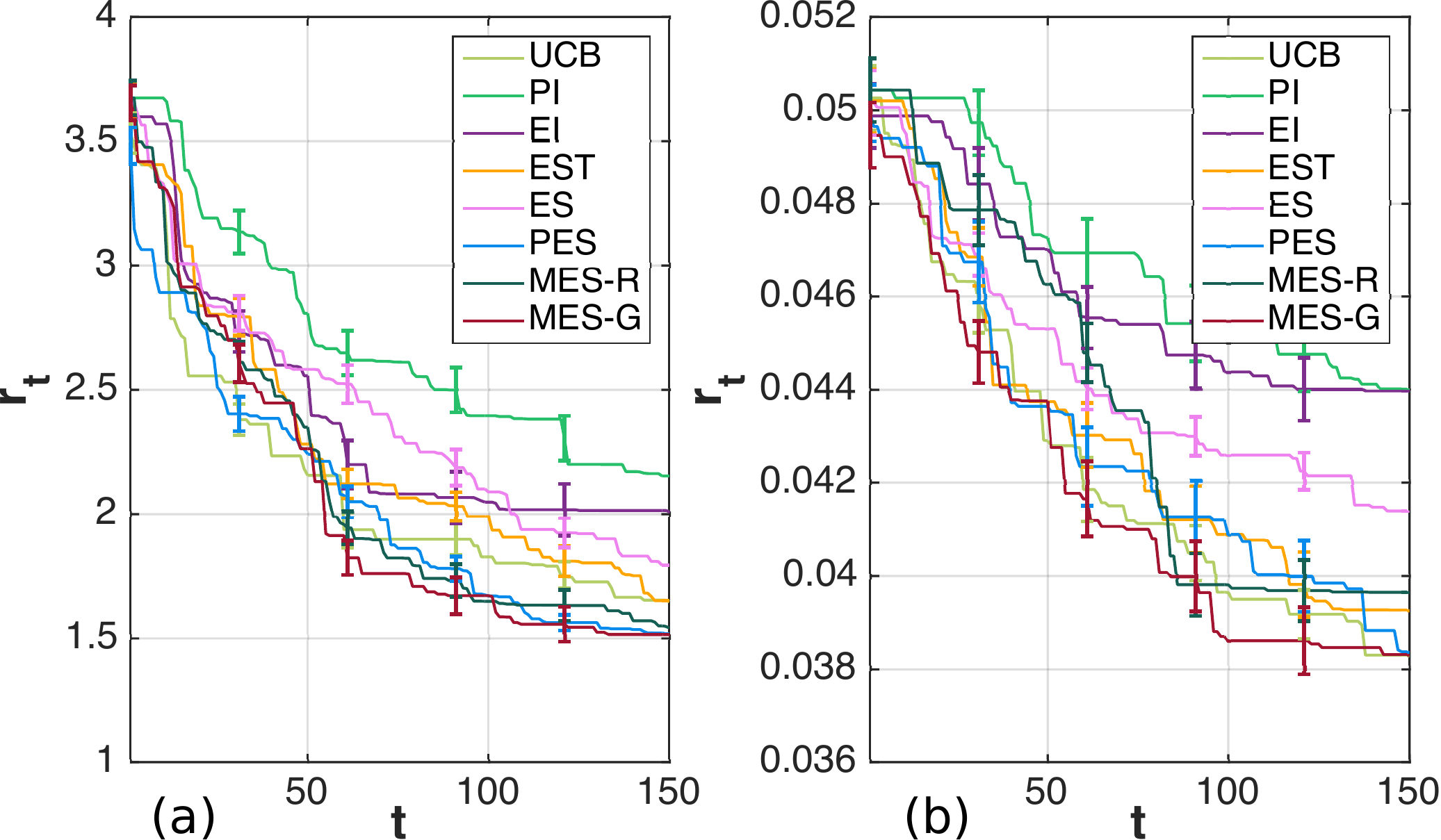}
\caption{Tuning hyper-parameters for training a neural network, (a) Boston housing dataset; (b) breast cancer dataset. MES methods perform better than other methods on (a), while for (b), MES-G, UCB, PES perform similarly and better than others.}
\label{fig:nnet}
\end{figure}

\begin{table}[t]
\caption{Inference regret $R_T$ for tuning neural network hyper-parameters on the Boston housing and breast cancer datasets.}
\label{tb:nnet}
\vskip -0.5in
\begin{center}
\begin{small}
\begin{sc}
\begin{tabular}{lll}%
\hline
\abovespace\belowspace
Method & Boston & Cancer (\%) \\
\hline
\abovespace
UCB & $1.64 \pm 0.43$ & {\color{red}$3.83 \pm 0.01$}  \\
PI & $2.15 \pm 0.99$ & $4.40 \pm 0.01$  \\
EI & $1.99 \pm 1.03$ & $4.40 \pm 0.01$  \\
EST & $1.65 \pm 0.57$ & $3.93 \pm 0.01$  \\
ES & $1.79 \pm 0.61$ & $4.14 \pm 0.00$  \\
PES &  {\color{red}$1.52 \pm 0.32$} &  {\color{red}$3.84 \pm 0.01$}  \\
MES-R & $1.54 \pm 0.56$ & $3.96 \pm 0.01$  \\
\belowspace
MES-G & {\color{red}$1.51 \pm 0.61$} & {\color{red}$3.83 \pm 0.01$}  \\
\hline
\end{tabular}
\end{sc}
\end{small}
\end{center}
\vskip -0.2in
\end{table}

Next, we experiment with Levenberg-Marquardt optimization for training a 1-hidden-layer neural network. The 4 parameters we tune with BO are the number of neurons,  the damping factor $\mu$, the $\mu$-decrease factor, and  the $\mu$-increase factor. We test regression on the Boston housing dataset and classification on the breast cancer dataset~\cite{bache2013uci}. The experiments are repeated 20 times, and the neural network's weight initialization and all other parameters are set to be the same to ensure a fair comparison. Both of the datasets were randomly split into train/validation/test sets. We initialize the observation set to have 10 random function evaluations which were set to be the same across all the methods. The averaged simple regret for the regression L2-loss on the validation set of the Boston housing dataset is shown in Fig.~\ref{fig:nnet}(a), and the classification accuracy on the validation set of the breast cancer dataset is shown in Fig.~\ref{fig:nnet}(b).   For the classification problem on the breast cancer dataset, MES-G, PES and UCB achieved a similar simple regret. On the Boston housing dataset, MES methods achieved a lower simple regret. We also show the inference regrets for both datasets in Table~\ref{tb:nnet}.

\subsection{Active Learning for Robot Pushing}
\begin{figure*}
\centering
\includegraphics[width=0.95\textwidth]{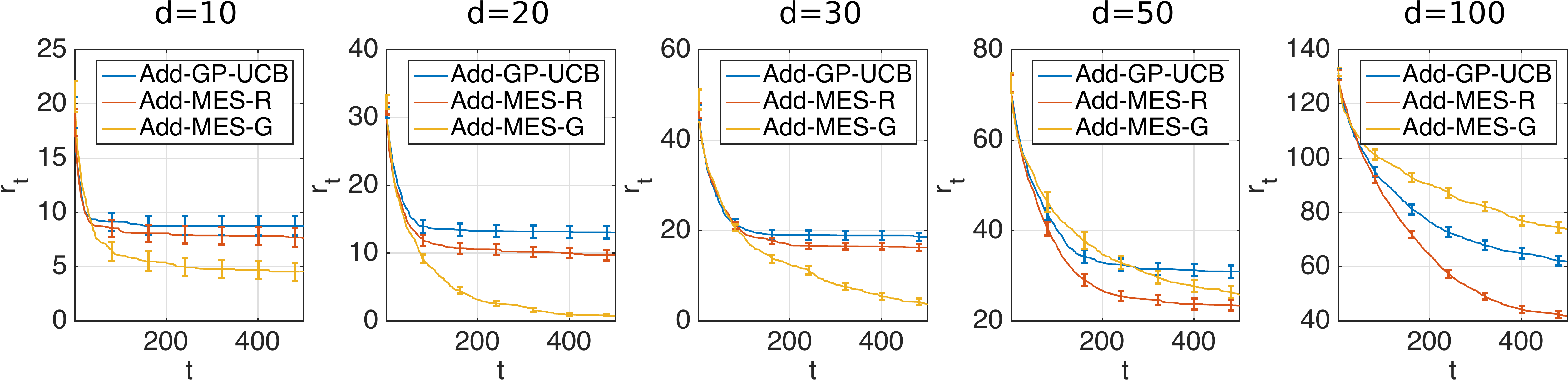}
\caption{Simple regrets for add-GP-UCB and add-MES methods on the synthetic add-GP functions. Both add-MES methods  outperform add-GP-UCB except for add-MES-G on the input dimension $d=100$. Add-MES-G achieves the lowest simple regret when $d$ is relatively low, while for higher $d$ add-MES-R becomes better than add-MES-G.}
\label{fig:highsynth}
\end{figure*}
\begin{figure}
\centering
\includegraphics[width=.45\textwidth]{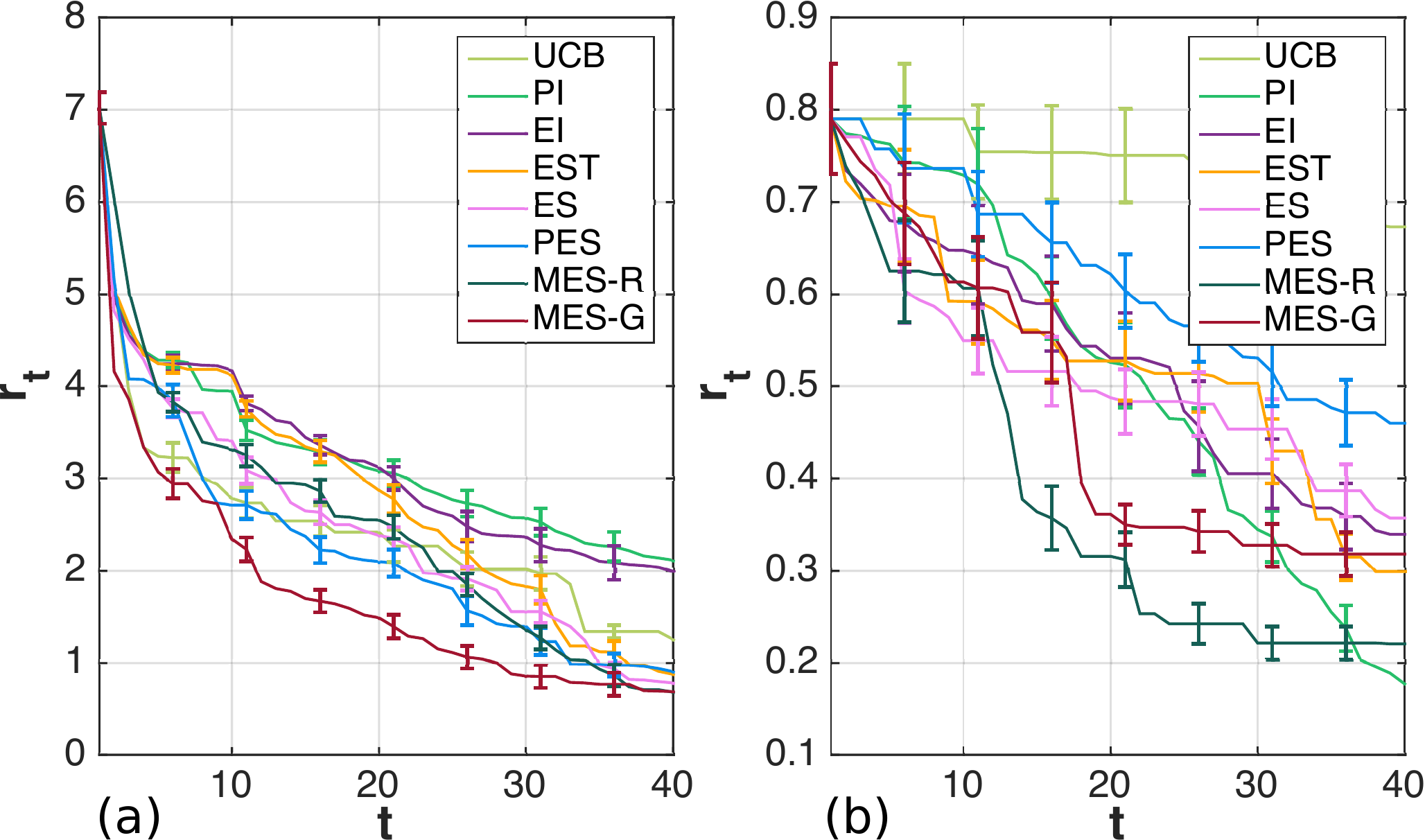}
\caption{BO for active data selection on two robot pushing tasks for minimizing the distance to a random goal with (a) 3-D actions and (b) 4-D actions. MES methods perform better than other methods on the 3-D function. For the 4-D function, MES methods converge faster to a good regret, while PI achieves lower regret in the very end.}
\label{fig:robot}
\end{figure}
We use BO to do active learning for the pre-image learning problem for pushing~\cite{kaelbling17icra}. The function we optimize takes as input the pushing action of the robot, and outputs the distance of the pushed object to the goal location. We use BO to minimize the function in order to find a good pre-image for pushing the object to the designated goal location. The first function we tested has a 3-dimensional input: robot location $(r_x,r_y)$ and pushing duration $t_r$. We initialize the observation size to be one, the same across all methods. The second function has a 4-dimensional input: robot location and angle $(r_x,r_y, r_{\theta})$, and pushing duration $t_r$. We initialize the observation to be 50 random points and set them the same for all the methods. We select 20 random goal locations for each function to test if BO can learn where to push for these locations. We show the simple regret in Fig.~\ref{fig:robot} and the inference regret in Table~\ref{tb:robot}. MES methods performed on a par with or better than their competitors.

\begin{table}[t]
\caption{Inference regret $R_T$ for action selection in robot pushing.}
\label{tb:robot}
\vskip -0.5in
\begin{center}
\begin{small}
\begin{sc}
\begin{tabular}{lll}%
\hline
\abovespace\belowspace
Method & 3-d action & 4-d action \\
\hline
\abovespace
UCB & $1.10 \pm 0.66$ & $0.56 \pm 0.44$  \\
PI & $2.03 \pm 1.77$ & {\color{red}$0.16 \pm 0.20$}  \\
EI & $1.89 \pm 1.87$ & $0.30 \pm 0.33$  \\
EST & $0.70 \pm 0.90$ & $0.24 \pm 0.17$  \\
ES &  {\color{red}$0.62 \pm 0.59$} & $0.25 \pm 0.20$  \\
PES & $0.81 \pm 1.27$ & $0.38 \pm 0.38$  \\
MES-R & {\color{red}$0.61 \pm 1.23$} & {\color{red}$0.16 \pm 0.10$}  \\
\belowspace
MES-G & {\color{red}$0.61 \pm 1.26$} & $0.24 \pm 0.25$  \\
\hline
\end{tabular}
\end{sc}
\end{small}
\end{center}
\vskip -0.2in
\end{table}

\subsection{High Dimensional BO with Add-MES}

In this section, we test our add-MES algorithm on high dimensional black-box function optimization problems. First we compare add-MES and add-GP-UCB~\cite{kandasamy2015high} on a set of synthetic additive functions with known additive structure and GP hyper-parameters. Each function component of the synthetic additive function is active on at most three input dimensions, and is sampled from a GP with zero mean and Gaussian kernel (bandwidth = $0.1$ and scale = $5$). For the parameter of add-GP-UCB, we follow~\cite{kandasamy2015high} and set $\beta^{(m)}_t = |A_m|\log 2t/5$. We set the number of $y^{(m)}_*$ sampled for each function component in add-MES-R and add-MES-G to be 1. 
We repeat each experiment for 50 times for each dimension setting. The results for simple regret are shown in Fig.~\ref{fig:highsynth}. Add-MES methods perform much better than add-GP-UCB in terms of simple regret. Interestingly, add-MES-G works better in lower dimensional cases where $d=10,20,30$, while add-MES-R outperforms both add-MES-G and add-GP-UCB for higher dimensions where $d=50,100$. In general, MES-G tends to overestimate the maximum of the
function because of the independence assumption, and MES-R tends to underestimate the maximum of the function because of the imperfect global optimization of the posterior function samples. We conjecture that MES-R is better for settings where exploitation is preferred over exploration (e.g., not too many local optima), and MES-G works better if exploration is preferred. 

To further verify the performance of add-MES in high dimensional problems, we test on two real-world high dimensional experiments. One is a function that returns the distance between a goal location and two objects being pushed by a robot which has 14 parameters\footnote{We implemented the function in~\cite{box2d}.}. The other function returns the walking speed of a planar bipedal robot, with 25 parameters to tune~\cite{westervelt2007feedback}. In Fig.~\ref{fig:highreal}, we show the simple regrets achieved by add-GP-UCB and add-MES. Add-MES methods performed competitively compared to add-GP-UCB on both tasks.

\begin{figure}[h]
\centering
\includegraphics[width=0.45\textwidth]{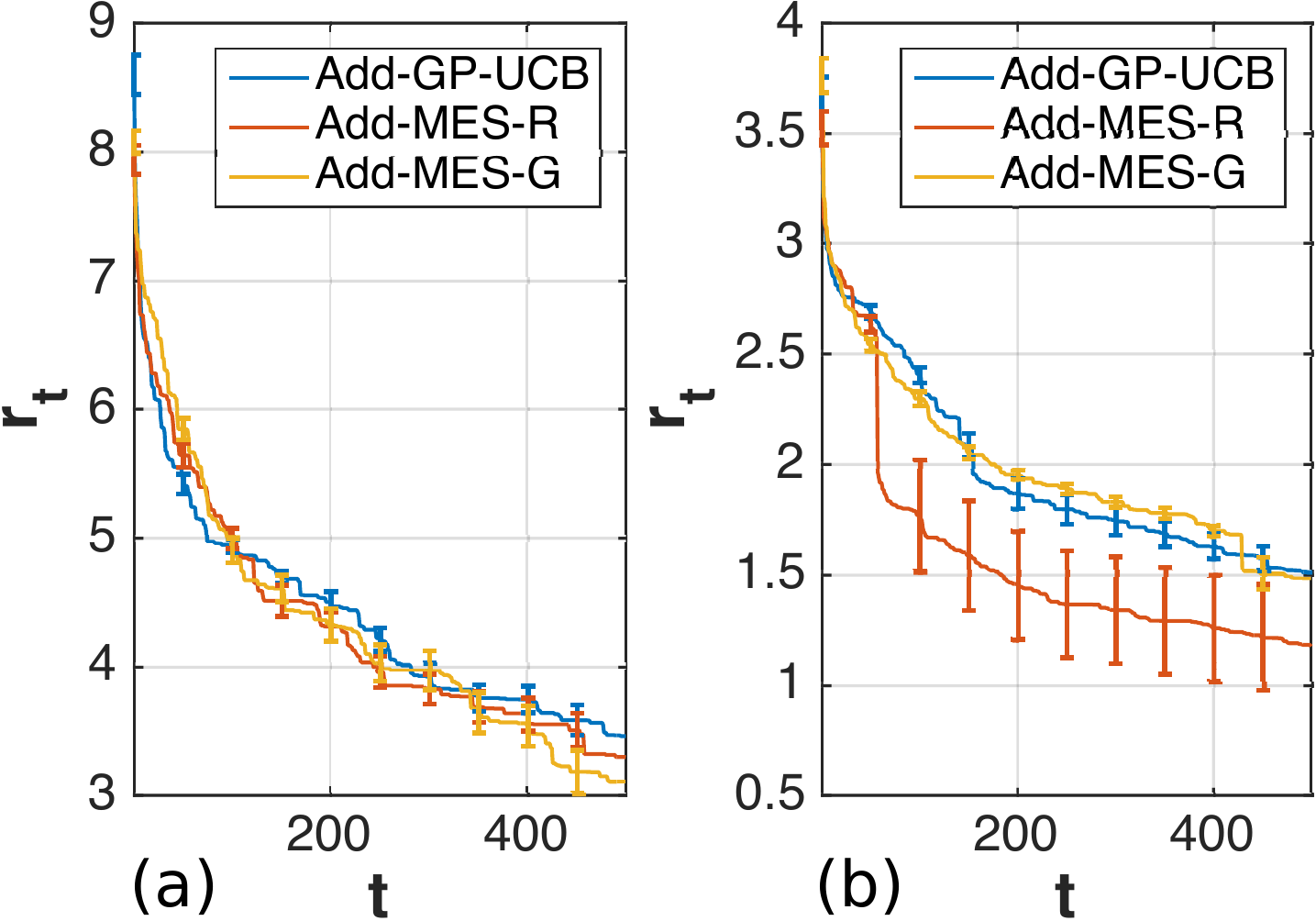}
\caption{Simple regrets for add-GP-UCB and add-MES methods on (a) a robot pushing task with 14 parameters and (b) a planar bipedal walker optimization task with 25 parameters. Both MES methods perform competitively comparing to add-GP-UCB.}
\label{fig:highreal}
\end{figure}

\section{Conclusion}
We proposed a new information-theoretic approach, max-value entropy search (MES), for optimizing expensive black-box functions. MES is competitive with or better than previous entropy search methods, but at a much lower computational cost. Via additive GPs, MES is adaptable to high-dimensional settings. We theoretically connected MES to other popular Bayesian optimization methods including entropy search, GP-UCB, PI, and EST, and showed a bound on the simple regret for a variant of MES. Empirically, MES performs well on a variety of tasks.

\section*{Acknowledgements}
We thank Prof.\ Leslie Pack Kaelbling and Prof.\ Tom\'as Lozano-P\'erez for discussions on active learning and Dr.\ William Huber for his solution to ``Extreme Value Theory - Show: Normal to Gumbel'' at \url{stats.stackexchange.com}, which leads to our Gumbel approximation in Section~\ref{spec:gumbel}. We gratefully acknowledge support from NSF CAREER award 1553284, NSF grants 1420927 and 1523767, from ONR grant N00014-14-1-0486, and from ARO grant W911NF1410433.  We thank MIT Supercloud and the Lincoln Laboratory Supercomputing Center for providing computational resources. Any opinions, findings, and conclusions or recommendations expressed in this material are those of the authors and do not necessarily reflect the views of our sponsors.

\begin{small}
\bibliography{refs}
\bibliographystyle{icml2017}
\end{small}

\appendix

\section{Related work}
Our work is largely inspired by the entropy search (ES) methods~\cite{hennig2012,hernandez2014predictive}, which established the information-theoretic view of Bayesian optimization by evaluating the inputs that are most informative to the $\argmax$ of the function we are optimizing. 

Our work is also closely related to probability of improvement (PI)~\cite{kushner1964}, expected improvement (EI)~\cite{mockus1974}, and  the BO algorithms using upper confidence bound to direct the search~\cite{auer2002b,kawaguchi2015bayesian,kawaguchi2016global}, such as  GP-UCB~\cite{srinivas2009gaussian}. In~\cite{wang2016est}, it was pointed out that GP-UCB and PI are closely related by exchanging the parameters. Indeed, all these algorithms build in the heuristic that the next evaluation point needs to be likely to achieve the maximum function value or have high probability of improving the current evaluations, which in turn, may also give more information on the function optima like how ES methods queries. These connections become clear as stated in Section 3.1 of our paper.

Finding these points that may have good values in high dimensional space is, however, very challenging. In the past, high dimensional BO algorithms were developed under various assumptions such as the existence of a lower dimensional function structure~\cite{djolonga2013high,wang2013}, or an additive function structure where each component is only active on a lower manifold of the space~\cite{li2016high,kandasamy2015high}. In this work, we show that our method also works well in high dimensions with the additive assumption made in~\cite{kandasamy2015high}.

\section{Using the Gumbel distribution to sample $y_*$}
To sample the function maximum $y_*$, our first approach is to approximate the distribution for $y^*$ and then sample from that distribution. We use independent Gaussians to approximate the correlated $f(\vx), \forall \vx\in \hat{\mathfrak X}$ where $\hat{\mathfrak X}$ is a discretization of the input search space $\mathfrak X$ (unless $\mathfrak X$ is discrete, in which case $\hat{\mathfrak X}=\mathfrak X$). A similar approach was adopted in~\cite{wang2016est}. We can show that by assuming $\{f(\vx)\}_{\vx\in \hat{\mathfrak X}}$, our approximated distribution gives a distribution for an upperbound on $f(\vx)$.

\begin{lem}[Slepian's Comparison Lemma~\citep{slepian1962one,massart2007concentration}]
\label{lem:splepian}
Let $\vu, \vv\in \R^n$ be two multivariate Gaussian random vectors with the same mean and variance, such that
 $$\mathbb E [\vv_i\vv_j]\leq \mathbb E [\vu_i\vu_j], \forall i,j.$$ 
 Then for every $y$
 $$\Pr[\sup_{i\in[1,n]}\vv_i \leq y] \leq \Pr[\sup_{i\in[1,n]}\vu_i \leq y].$$
\end{lem}

By the Slepian's lemma, if the covariance $k_t(\vx, \vx') \geq 0, \forall \vx,\vx'\in  \hat{\mathfrak X}$, using the independent assumption with give us a distribution on the upperbound $\hat{y}_*$ of $f(\vx)$, $\Pr[\hat{y}_* < y] = \prod_{\vx\in\hat{\mathfrak X}} \Psi(\gamma_y(\vx)))$.

We then use the  Gumbel distribution to approximate the distribution for the maximum of the function values for $\hat{\mathfrak X}$, $\Pr[\hat{y}_* < y] = \prod_{\vx\in\hat{\mathfrak X}} \Psi(\gamma_y(\vx)))$. If for all $\vx\in\hat{\mathfrak X}$, $f(\vx)$ have the same mean and variance, the Gumbel approximation is in fact asymptotically correct by the Fisher-Tippett-Gnedenko theorem~\cite{fisher1930genetical}. 
\begin{thm}[The Fisher-Tippett-Gnedenko Theorem~\cite{fisher1930genetical}]
Let $\{v_i\}_{i=1}^\infty$ be a sequence of independent and identically-distributed random variables, and $M_n = \max_{1\leq i\leq n}v_i$. If there exist constants $a_n>0,b_n \in \R$ and a non degenerate distribution function $F$ such that $\lim_{n\rightarrow \infty} \Pr(\frac{M_n-b_n}{a_n}\leq x) = F(x)$, then the limit distribution $F$ belongs to either the Gumbel, the Fr{\'e}chet or the Weibull family.
\end{thm}
In particular, for i.i.d. Gaussians, the limit distribution of the maximum of them belongs to the Gumbel distribution~\cite{von1936distribution}.
 Though the Fisher-Tippett-Gnedenko theorem does not hold for independent and differently distributed Gaussians, in practice we still find it useful in approximating $\Pr[\hat{y}_* < y]$. In Figure~\ref{fig:gumbel}, we show an example of the result of the approximation for the distribution of the maximum of $f(\vx)\sim GP(\mu_t,k_t) \forall \vx\in\hat{\mathfrak X}$ given 50 observed data points randomly selected from a function sample from a GP with 0 mean and Gaussian kernel. 

\begin{figure}[h]
\centering
\includegraphics[width=.4\textwidth]{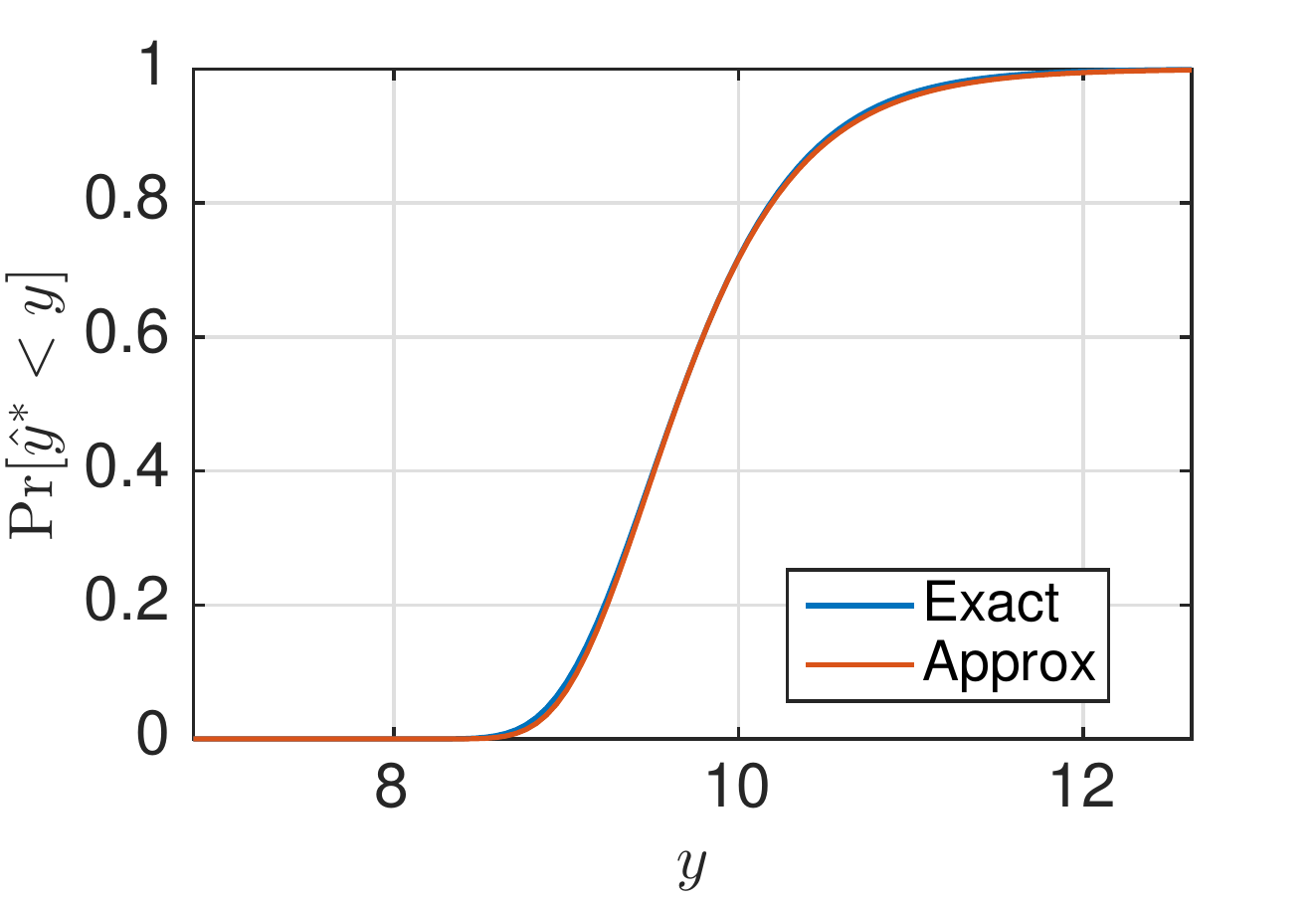}
\caption{An example of approximating the cumulative probability of the maximum of independent differently distributed Gaussians $\Pr[\hat{y}_* < y]$ (Exact) with a Gumbel distribution $\mathcal G(a,b)$ (Approx) via percentile matching. }
\label{fig:gumbel}
\end{figure}

\hide{
\textbf{Speeding up posterior inference. } The random features not only enable easy sampling of a function from the posterior, but also speed up the process of updating the GP when a new observation is added: $\mSigma_t$ can be updated in $O(D^2)$ by the Sherman-Morrison formula: %
$\mSigma_{t+1} = \mSigma_{t} - \frac{\mSigma_{t} z_{t+1} z_{t+1}\T \mSigma_{t}\sigma^{-2}}{1+ z_{t+1}\T \mSigma_{t} z_{t+1} \sigma^{-2}}$, which is cheaper than  
the inversion of a $t\times t$ matrix for every iteration $t$.
Similarly, to update $k_t(\cdot, \cdot)$, we get 
$
k_{t+1}(\vx, \vx') = k_t(\vx, \vx') - \frac{k_t(\vx, \vx_{t+1}) k_t(\vx_{t+1}, \vx')}{\sigma^2 +k_t(\vx_{t+1}, \vx_{t+1})}.
$
}
\section{Regret bounds}
Based on the connection of MES to EST, we show the bound on the learning regret for MES with a point estimate for $\alpha(x)$. 
\regret*

\begin{figure*}
\centering
\includegraphics[width=0.85\textwidth]{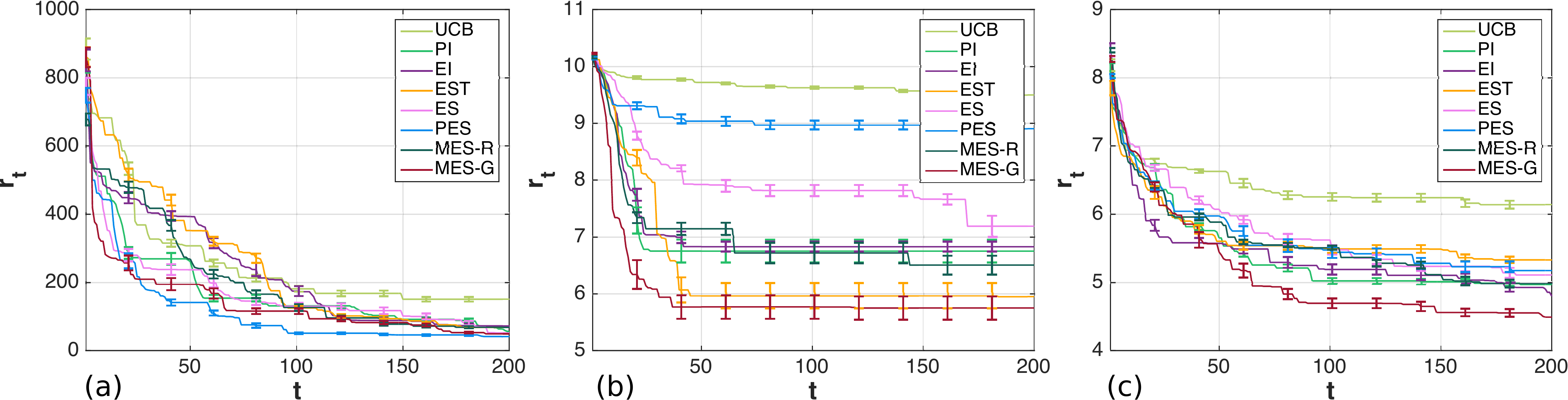}
\caption{(a) 2-D eggholder function; (b) 10-D Shekel function; (c) 10-D Michalewicz function. PES achieves lower regret on the 2-d function while MES-G performed better than other methods on the two 10-d optimization test functions.}
\label{fig:test}
\end{figure*}

Before we continue to the proof, notice that if the function upper bound $\hat y_*$ is sampled using the approach described in Section 3.1 and $k_t(\vx, \vx') \geq 0, \forall \vx,\vx'\in  \hat{\mathfrak X}$, we may still get the regret guarantee by setting $y_*=\hat{y}_*$ (or $y_* = \hat{y}_* +\epsilon L$ if $\mathfrak X$ is continuous) since $\Pr[\max_{\hat{\mathfrak X}} \leq y] \geq \Pr[\hat{y}_* < y]$. Moreover, Theorem~\ref{thm:regret} assumes $y_*$ is sampled from a universal maximum distribution of functions from $GP(\mu,k)$, but it is not hard to see that if we have a distribution of maximums adapted from $GP(\mu_t, k_t)$, we can still get the same regret bound by setting $T' = \sum_{i=1}^T\log_{w_i} \frac{\delta}{2\pi_i}$, where $w_i = F_i(f_*)$ and $F_i$ corresponds to the maximum distribution at an iteration where $y_*>f_*$. Next we introduce a few lemmas and then prove Theorem~\ref{thm:regret}.

\begin{lem}[Lemma 3.2 in \cite{wang2016est}]
\label{lem:pbound}
Pick $\delta\in(0,1)$ and set $\zeta_{t} = (2\log(\frac{\pi_t}{2\delta}))^\frac12$, where $\sum_{t=1}^T \pi_t^{-1} \leq 1$, $\pi_t > 0$. Then, it holds that
$ \Pr [  \mu_{t-1}(\vx_{t}) -f( \vx_{t})  \leq \zeta_{t}\sigma_{t-1}(\vx_{t}), \forall t\in [1,T]] \geq 1-\delta$.
\end{lem}
\begin{lem}[Lemma 3.3 in \cite{wang2016est}]
\label{regretlem}
If $ \mu_{t-1}(\vx_{t}) -f(\vx_{t}) \leq \zeta_t\sigma_{t-1}(\vx_{t})$, the regret at time step $t$ is upper bounded as $\rt_t \leq (\bt_t  +\zeta_t )\sigma_{t-1}(\vx_{t})$ %
, where $\bt_t  \triangleq \min_{\vx\in\mathfrak X} \frac{\hat m_t-\mu_{t-1}(\vx)}{\sigma_{t-1}(\vx)}$, and $\hat m_t\geq \max_{\vx\in \mathfrak X} f(\vx)$, $\forall t\in[1,T]$.
\end{lem}
\begin{lem}[Lemma 5.3 in \cite{srinivas2009gaussian}]
\label{lem:info}
The information gain for the points selected can be expressed in terms of the predictive variances. If $f_T = (f(\vx_t))\in \R^T$:
$$I(\vy_T; \vf_T) = \frac12 \sum_{t=1}^T \log(1+\sigma^{-2}\sigma^2_{t-1}(\vx_t)).$$
\end{lem}
\begin{proof}(Theorem~\ref{thm:regret})
By lemma~3.1 in our paper, we know that the theoretical results from EST~\cite{wang2016est} can be adapted to MES if $y_* \geq f_*$. The key question is when a sampled $y_*$ that can satisfy this condition. Because the cumulative density $w = F(f_*) \in (0,1)$ and $y^t_*$ are  independent samples from $F$, there exists at least one $y^t_*$ that satisfies $y^t_* > f_*$ with probability at least $1-w^{k_i}$ in $k_i$ iterations. 

Let $T' = \sum_{i=1}^T k_i$ be the total number of iterations. We split these iterations to $T$ parts where each part have $k_i$ iterations, $i = 1,\cdots, T$. By union bound, with probability at least $1-\sum_{i=1}^T w^{k_i}$, in all the $T$ parts of iterations, we have at least one iteration $t_i$ which samples $y^{t_i}_*$ satisfying $y^{t_i}_*>f_*, \forall i = 1,\cdots, T$. 

Let $\sum_{i=1}^T w^{k_i} = \frac{\delta}{2}$, we can set $k_i = \log_w \frac{\delta}{2\pi_i}$ for any $\sum_{i=1}^T(\pi_i)^{-1} = 1$. A convenient choice for $\pi_i$ is $\pi_i= \frac{\pi^2 i^2}{6}$. Hence with probability at least $1-\frac{\delta}{2}$, there exist a sampled $y^{t_i}_*$ satisfying $y^{t_i}_*>f_*, \forall i = 1,\cdots, T$.

Now let $\zeta_{t_i} = (2\log\frac{\pi_{t_i}}{\delta})^{\frac12}$. By Lemma~\ref{lem:pbound} and Lemma~\ref{regretlem}, the immediate regret $r_{t_i} = f_* - f(\vx_{t_i})$ can be bounded as
\begin{align*}
r_{t_i} &\leq (\bt_{t_i}  +\zeta_{t_i} )\sigma_{t_i-1}(\vx_{t_i}).
\end{align*}
Note that by assumption $0\leq\sigma_{t_i-1}^2(\vx_{t_i}) \leq 1$, so we have $\sigma_{t_i-1}^2 \leq \frac{\log(1+\sigma^{-2} \sigma_{t_i-1}^2(\vx_{t_i}))}{\log(1+\sigma^{-2})}$.
Then by Lemma~\ref{lem:info}, we have $\sum_{i=1}^T \sigma_{t_i-1}^2(\vx_{t_i}) \leq \frac{2}{\log(1+\sigma^{-2})} I(\vy_T;\vf_T)$ where $\vf_T = (f(\vx_{t_i}))_{i=1}^T \in \R^T, \vy_T = (y_{t_i})_{i=1}^T \in \R^T$. From assumptions, we have $I(\vy_T;\vf_T)\leq\rho_T $. By Cauchy-Schwarz inequality, $\sum_{i=1}^T \sigma_{t_i-1}(\vx_{t_i})\leq \sqrt{T\sum_{i=1}^T \sigma^2_{t_i-1}(\vx_{t_i})} \leq \sqrt{\frac{2T \rho_T}{\log(1+\sigma^{-2})}}$. It follows that with probability at least $1-\delta$,
$$\sum_{i=1}^T r_{t_i} \leq  (\bt_{t^*}  +\zeta_{T} ) \sqrt{\frac{2T \rho_T}{\log(1+\sigma^{-2})}}.$$
As a result, our learning regret is bounded as
$$r_{T'}\leq\frac1T\sum_{i=1}^T r_{t_i} \leq  (\bt_{t^*}  +\zeta_{T} ) \sqrt{\frac{2 \rho_T}{T\log(1+\sigma^{-2})}},$$ where $T' = \sum_{i=1}^T k_i = \sum_{i=1}^T \log_{w}\frac{\delta}{2\pi_i}$ is the total number of iterations.
\end{proof}
At first sight, it might seem like MES with a point estimate does not have a converging rate as good as $EST$ or $GP-UCB$. However, notice that $\min_{\vx\in\mathfrak X}\gamma_{y_1}(\vx) < \min{\vx\in\mathfrak X} \gamma_{y_2}(\vx)$ if $y_1 < y_2$, which decides the rate of convergence in Eq.~\ref{eq:regret}. So if we use $y_*$ that is too large, the regret bound could be worse. If we use $y_*$ that is smaller than $f_*$, however, its value won't count towards the learning regret in our proof, so it is also bad for the regret upper bound. With no principled way of setting $y_*$ since $f_*$ is unknown. Our regret bound in Theorem~\ref{thm:regret} is a randomized trade-off between sampling large and small $y_*$.

For the regret bound in add-GP-MES, it should follow add-GP-UCB. However, because of some technical problems in the proofs of the regret bound for add-GP-UCB, we haven't been able to show a regret bound for add-GP-MES either. Nevertheless, from the experiments on high dimensional functions, the methods worked well in practice.
\section{Experiments}

In this section, we provide more details on our experiments.

\paragraph{Optimization test functions}
In Fig.~\ref{fig:test}, we show the simple regret comparing BO methods on the three challenging optimization test functions: the 2-D eggholder function, the 10-D Shekel function, and the 10-D Michalewicz function.

\paragraph{Choosing the additive decomposition} 
We follow the approach in~\cite{kandasamy2015high}, and sample
10000 random decompositions (at most 2 dimensions in each
group) and pick the one with the best data likelihood based on 500
data points uniformly randomly sampled from the search space.
The decomposition setting was fixed for all the 500 iterations of
BO for a fair comparison.

\end{document}